\documentclass[11pt,letterpaper]{article}
\usepackage[margin=1in]{geometry}
\usepackage{lmodern}
\usepackage[round,authoryear]{natbib}
\usepackage[utf8]{inputenc}
\usepackage[T1]{fontenc}
\usepackage[hidelinks,hypertexnames=false]{hyperref}
\usepackage{url}
\usepackage{booktabs}
\usepackage{graphicx}
\usepackage{amsfonts}
\usepackage{nicefrac}
\usepackage{microtype}
\usepackage{xcolor}
\newcommand{\rev}[1]{#1}

\newcommand{\redrev}[1]{#1}
\newenvironment{redrevision}{\par\begingroup}{\par\endgroup}
\usepackage{amsmath,amssymb,amsthm,mathtools}
\usepackage{enumitem}
\usepackage{array}
\usepackage{tabularx}
\usepackage{float}

\newcommand{\E}{\mathbb{E}}
\newcommand{\PP}{\mathbb{P}}
\newcommand{\R}{\mathbb{R}}

\newcommand{\1}{\mathbf{1}}
\newcommand{\dd}{\,\mathrm{d}}
\newcommand{\Bin}{\mathrm{Bin}}
\newcommand{\Mult}{\mathrm{Mult}}
\newcommand{\Bern}{\mathrm{Bernoulli}}
\newcommand{\supp}{\operatorname{supp}}

\newcommand{\cP}{\mathcal{P}}

\newcommand{\cB}{\mathcal{B}}

\newcommand{\sym}{\mathrm{sym}}
\newcommand{\Vinf}{V_\infty}
\newcommand{\cG}{\mathcal{G}}

\theoremstyle{plain}
\newtheorem{theorem}{Theorem}[section]
\newtheorem{proposition}[theorem]{Proposition}
\newtheorem{corollary}[theorem]{Corollary}
\newtheorem{lemma}[theorem]{Lemma}
\theoremstyle{definition}
\newtheorem{definition}[theorem]{Definition}
\newtheorem{example}[theorem]{Example}
\theoremstyle{remark}

\title{When Can Voting Help, Hurt, or Change Course? Exact Structure of Binary Test-Time Aggregation}
\author{Yi Liu\\[3pt] {\normalsize York University}}
\date{}

\begin{document}
\raggedbottom
\maketitle

\begin{abstract}
\redrev{Increasing a voting budget can improve or reduce population accuracy, and an initial gain or loss need not persist. We give a complete structural classification of these curves for binary unweighted majority under infinitely exchangeable correctness indicators, with an arbitrary latent success-probability law \(\Pi\). The classification is encoded by the signed voting signature}
\[
\omega_\Pi(B)=\int \1\{q(1-q)\in B\}(2q-1)\,\Pi(\dd q).
\]
\redrev{The complete odd-budget curve and this signature determine each other: the initial accuracy gives its total signed mass, and successive budget increments give its higher moments. Reflection-symmetric differences between latent laws are exactly the changes invisible to voting. A weighted total-variation constraint characterizes every attainable signature, and a minimum-mass allocation to the two reflected branches, completed by symmetric probability mass, gives every population producing its curve. Consequently, the full curve identifies the latent population precisely when no symmetric mass remains to redistribute.}
\end{abstract}

\section{Introduction}
\label{sec:intro}

\redrev{Majority voting is often expected to improve prediction accuracy. However, as the voting budget grows, population accuracy can rise, fall, remain constant, or change direction. We give a complete structural classification of these budget-dependent curves for binary unweighted majority under infinitely exchangeable correctness indicators. Repeated calls to a language model or black-box predictor under fresh randomization provide this additional computation without retraining the model. On multiple-choice MMLU Physics, \citet{chen2024compound} observed that GPT-3.5 voting accuracy first rose and then fell with more calls; \citet{khan2024debating} reported a slight decline as GPT-4-Turbo judge votes increased in a binary consultancy task.}

\redrev{Let an item's success probability be \(Q\sim\Pi\), and let \(P_n(q)\) denote the accuracy of a majority of \(2n+1\) conditionally independent binary predictions at success probability \(q\). The population accuracy is}
\[
V_n(\Pi)=\E\{P_n(Q)\}.
\]
\redrev{For a fixed item, repetition favors whichever answer is already more likely: the correct answer above \(q=1/2\), and the wrong answer below. A common budget can thus help one part of the population while hurting another. These contributions evolve at different rates \citep{chen2024compound}; a favorable average \(V_0=\E Q>1/2\) does not determine which contribution will dominate later.}

Figure~\ref{fig:shape-examples} shows how different the resulting curves can be.  All five curves have one-vote accuracy \(V_0=3/4\), yet ordinary majority voting can remain constant, fall quickly, dip and later surpass its starting value, rise and then fall, or rise slowly.  \redrev{Each curve is generated by an explicit latent law under the same binary majority rule.}

\begin{figure}[t]
    \centering
    \includegraphics[width=0.84\linewidth]{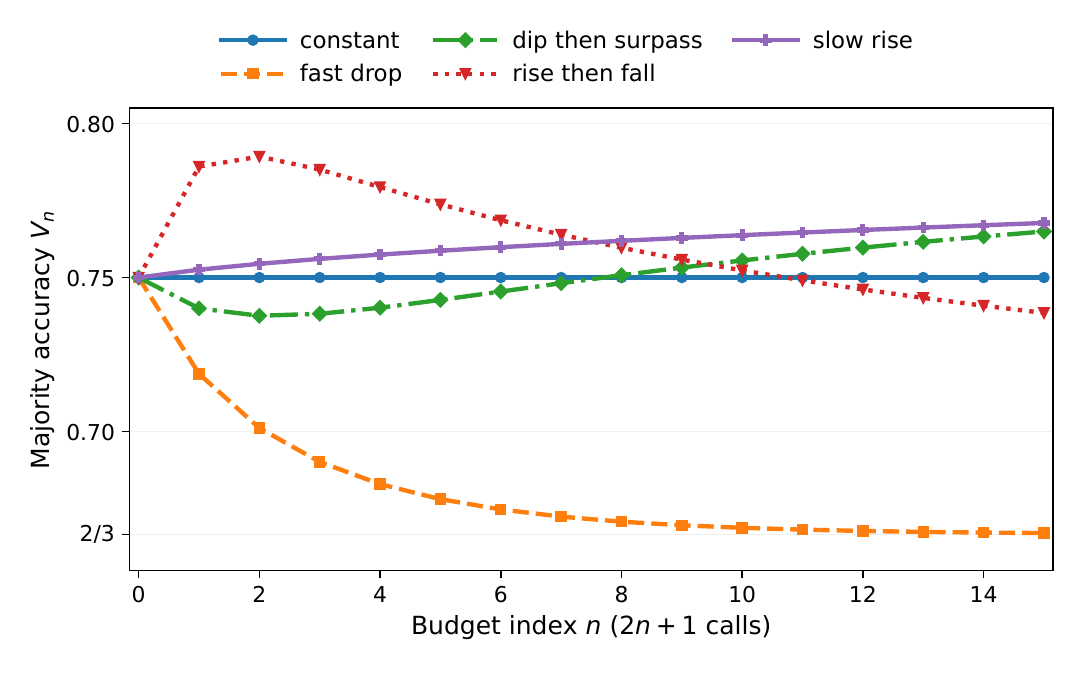}
    \caption{\redrev{Five exact binary majority-vote curves with the same one-call accuracy \(V_0=3/4\). Increasing the budget can preserve accuracy, reduce it, improve it slowly, or reverse an earlier trend. A budget index \(n\) uses \(2n+1\) calls. Appendix~\ref{app:figure1-laws} gives the latent laws and establishes their shapes beyond the plotted range.}}
    \label{fig:shape-examples}
\end{figure}

\redrev{An example of a reversal shows that monotonicity can fail. A classification of the complete curve answers more: it describes every attainable pattern of budget-dependent accuracy and all populations that produce the same pattern. This also supplies the reference point for finite evaluations. Before extrapolating from a few budgets, one needs to distinguish information that the full curve would identify from information that voting discards even with unlimited access.}

\redrev{Our starting observation is that success probabilities \(q\) and \(1-q\) are equally far from chance and have opposite effects on voting. Equal mass on this pair has chance accuracy at every budget. We record the imbalance that remains in a \emph{signed voting signature}. As the budget grows, the increments increasingly emphasize imbalances nearer chance, explaining how the population trend can reverse while each item's trend remains fixed.}

\redrev{We prove that the full curve and this signature determine each other, and then characterize which signatures are compatible with a probability law and every law realizing them (Sections~\ref{sec:signed-signature}--\ref{sec:curve-realizability}). This turns cancellation into an exact classification, including the criterion for a curve to identify a unique population. The resulting coordinates explain monotonicity, infinitely many reversals, and finite total variation (Section~\ref{sec:voting-shapes}), separate identifiable finite-repeat information from extrapolation (Section~\ref{sec:finite-repeat-access}), and determine the endpoint and its convergence rates (Section~\ref{sec:variation-shape}).}

\section{Related work}
\label{sec:related}

\redrev{Self-consistency aggregates sampled reasoning paths by voting \citep{wang2023}, and repeated sampling provides a test-time compute axis \citep{brown2024}. Process verification and test-time scaling likewise study how additional inference computation changes performance \citep{lightman2024,muennighoff2025}. Repeated predictions also support uncertainty and consistency assessment in language models and ensembles \citep{manakul2023,kuhn2023,farquhar2024,gal2016,lakshminarayanan2017,xiao2025}. We study the accuracy of the unweighted binary majority rule as its budget varies, holding the stochastic predictor fixed.}

\redrev{Recently, \citet{chen2024compound} showed how heterogeneous query difficulties can produce nonmonotonic majority-vote accuracy and characterized monotone and reversing regimes in a two-type binary mixture. Their forward formula also averages itemwise majority accuracy over an arbitrary query population. We address the inverse and realizability questions: exactly what population information does an entire curve identify, which such objects are attainable, and which populations share them? The answer covers all latent laws on \([0,1]\), including curves with infinitely many reversals.}

\redrev{Classical jury theorems give the fixed-competence benchmark; see \citet{dietrich2021}. Here heterogeneity is across items: repeated calls on one item share its latent success probability. This differs from assigning different fixed competences to voters answering a common question. Infinite exchangeability formalizes our item-level random-effects model through de Finetti's representation \citep{diaconis1980,dolera2020}.}

\redrev{For finite repeated observations, Bernoulli-mixture methods study recovery of mixing laws from grouped binomial counts \citep{tian2017,vinayak2019,polyanskiy2020,zhang2024}. Our target is the part of a mixing law that determines the voting curve. With two labeled calls, \citet{liu2026twocalls} derives sharp partial-identification bounds for selected finite-budget accuracies. We characterize the signature moments identified at every fixed repeat depth, give their estimators, and establish nonidentification of each higher moment over the unrestricted class.}

\section{\redrev{From Budget Changes to a Population Signature}}
\label{sec:binary-model}
\label{sec:signed-signature}

\redrev{Consider a randomly drawn binary-labeled item and correctness bits \(B_1,B_2,\ldots\) from a fixed stochastic binary predictor. We assume infinite exchangeability: permuting any finite set of calls does not change their joint law. De Finetti's representation then gives a latent variable \(Q\sim\Pi\), with a law on \([0,1]\) shared across all repeat budgets, such that}
\[
B_j\mid Q\sim\Bern(Q)\quad\text{iid}.
\]
For odd budget \(2n+1\), define
\[
P_n(q)=\sum_{j=n+1}^{2n+1}\binom{2n+1}{j}q^j(1-q)^{2n+1-j},
\qquad
V_n(\Pi)=\int P_n(q)\,\Pi(\dd q),
\]
with \(P_0(q)=q\).
Under fair tie-breaking, an even budget gives the same accuracy as the preceding odd budget; Appendix~\ref{app:additional-structure} states the exact identity. Hence the odd-budget curve records all distinct accuracies of ordinary unweighted majority.

\redrev{Subtracting the majority probabilities at successive odd budgets gives, for every \(n\ge0\) and \(q\in[0,1]\),}
\begin{equation}
P_{n+1}(q)-P_n(q)
=\binom{2n+1}{n+1}\{q(1-q)\}^{n+1}(2q-1).
\label{eq:pointwise-diff}
\end{equation}
\redrev{Only a change at the majority boundary contributes to this difference: adding two calls can turn \(n\) correct votes into a correct majority or \(n+1\) correct votes into an incorrect majority. The identity also appears in \citet[Lemma~6]{chen2024compound}; Appendix~\ref{app:binary-proofs} gives the calculation. For fixed \(1/2<q<1\), \(P_n(q)\) increases with diminishing increments; for \(0<q<1/2\), it decreases with diminishing losses. At \(q\in\{0,1/2,1\}\), it is constant.}

\redrev{The factor \(2q-1\) fixes the direction of an item's contribution. The power of \(q(1-q)\) determines how long that contribution persists as the budget grows. Values nearer \(1/2\) have larger \(q(1-q)\), so they receive increasing \emph{relative} weight in later increments. A reversal can occur when the population imbalance favors correct predictions at one distance from the threshold and incorrect predictions at another.}

\redrev{This suggests grouping probabilities by their distance from chance. Write}
\[
R=Q(1-Q),\qquad r=q(1-q)\in[0,1/4],\qquad
q_\pm(r)=\frac{1\pm\sqrt{1-4r}}{2}.
\]
\redrev{The two inverse functions \(q_+\) and \(q_-\) are the \emph{reflected branches}: they pair equally distant probabilities above and below \(1/2\). Their orientation factors are \(2q_\pm(r)-1=\pm\sqrt{1-4r}\). Since \(P_n(1-q)=1-P_n(q)\), equal mass on a reflected pair contributes only chance accuracy, at every budget. A signed measure on the common scale \(r\) records what remains after this cancellation.}

\begin{definition}
\label{def:omega}
For \(\Pi\in\cP([0,1])\), define the signed Borel measure on \([0,1/4]\)
\[
\omega_\Pi(B)
=\int_{[0,1]}\1\{q(1-q)\in B\}(2q-1)\,\Pi(\dd q),
\qquad B\in\cB([0,1/4]).
\]
We call \(\omega_\Pi\) the \emph{signed voting signature}.
\redrev{Two laws are \emph{voting-equivalent}, written \(\Pi_1\sim\Pi_2\), when \(V_n(\Pi_1)=V_n(\Pi_2)\) for every \(n\ge0\).}
\end{definition}

\redrev{The signature retains exactly the information in the complete curve. To state this equivalence, let \(\iota(q)=1-q\) and let \(\iota_\#\nu\) denote the reflection of a measure \(\nu\). The next theorem also identifies every redistribution of population mass that leaves the curve unchanged.}

\begin{theorem}[Curve--signature equivalence]
\label{thm:omega-moments}
For every latent law \(\Pi\),
\begin{align}
V_0(\Pi)
&=\frac12\{1+\omega_\Pi([0,1/4])\},
\label{eq:v0-omega}\\
V_{n+1}(\Pi)-V_n(\Pi)
&=\binom{2n+1}{n+1}\int_{[0,1/4]}r^{n+1}\,\omega_\Pi(\dd r),
\qquad n\ge0.
\label{eq:delta-omega}
\end{align}
\redrev{The complete odd-budget curve and its signature determine each other uniquely. Moreover,}
\begin{redrevision}
\[
\Pi_1\sim\Pi_2
\quad\Longleftrightarrow\quad
\omega_{\Pi_1}=\omega_{\Pi_2}
\quad\Longleftrightarrow\quad
\iota_\#(\Pi_1-\Pi_2)=\Pi_1-\Pi_2.
\]
\end{redrevision}
\redrev{Thus the signature is a complete invariant of voting equivalence: precisely the reflection-symmetric differences between probability laws are invisible to the curve.}
\end{theorem}

The inverse map is explicit:
\begin{corollary}
\label{cor:moment-recovery}
\[
\int\omega_\Pi(\dd r)=2V_0(\Pi)-1,\qquad
\int r^k\,\omega_\Pi(\dd r)
=\frac{V_k(\Pi)-V_{k-1}(\Pi)}{\binom{2k-1}{k}},
\quad k\ge1.
\]
\end{corollary}

\redrev{The initial accuracy gives the signature's total signed mass, and each subsequent increment gives one higher moment. On the compact interval \([0,1/4]\), these moments determine a finite signed measure: equality against polynomials extends, by uniform approximation, to all continuous test functions. Appendix~\ref{app:binary-proofs} proves this identification and the characterization of the invisible, reflection-symmetric differences. Reflecting an entire law negates its signature and complements its curve.}

\redrev{The zeroth moment is essential: the imbalance between \(q=1\) and \(q=0\) maps to \(r=0\), affecting the starting accuracy but no increment. Mass at \(q=1/2\) has zero orientation and contributes only the chance baseline.}

\section{\redrev{Attainable Curves and Their Populations}}
\label{sec:curve-realizability}

\redrev{Completing the classification requires determining which signatures can come from a probability law and all laws that realize each one.}

\subsection{\redrev{Density Coordinates and Admissibility}}
\label{sec:density-admissibility}

\redrev{For a population density \(f\), the signature has a particularly direct form. The inverse-branch Jacobian \((1-4r)^{-1/2}\) cancels the orientation magnitude, giving}
\begin{redrevision}
\[
\omega_\Pi(\dd r)=g(r)\dd r,
\qquad g(r)=f(q_+(r))-f(q_-(r)).
\]
\end{redrevision}
\redrev{Thus \(g\) is exactly the reflected-branch imbalance. The same geometry covers atoms and mixtures: at scale \(r<1/4\), one unit of probability contributes at most \(\sqrt{1-4r}\) units of signed mass. Realizing an imbalance therefore has a minimum probability-mass cost.}

\redrev{The same coordinates express each accuracy level, not only the increments. For \(r<1/4\) write}
\begin{redrevision}
\[
K_n(r)=\frac{2P_n(q_+(r))-1}{\sqrt{1-4r}},
\]
\end{redrevision}
\redrev{the majority advantage of an upper-branch item divided by its orientation. The apparent singularity at \(r=1/4\) is removable: \(K_n(r)\to P_n'(1/2)=(2n+1)\binom{2n}{n}4^{-n}\).}
\begin{redrevision}
\begin{theorem}[Reflected-branch density coordinates]
\label{thm:density-rep}
If \(Q\) has density \(f\), then \(\omega_\Pi(\dd r)=g(r)\dd r\), so that
\(V_{n+1}-V_n=\binom{2n+1}{n+1}\int_0^{1/4}r^{n+1}g(r)\dd r\), and, with \(K_n\) extended continuously to \(r=1/4\),
\begin{equation}
V_n=\frac12+\frac12\int_0^{1/4}K_n(r)g(r)\,\dd r.
\label{eq:g-level}
\end{equation}
\end{theorem}
\end{redrevision}
\redrev{As \(n\to\infty\), \(K_n(r)\) increases to \(1/\sqrt{1-4r}\), the weight in the endpoint formula of Section~\ref{sec:variation-shape}. Appendix~\ref{app:branch-proofs} contains the proof.}

\redrev{The \emph{branch lift} realizes positive signature contributions on the upper branch and negative ones on the lower, using their magnitudes divided by \(\sqrt{1-4r}\). The remaining probability can be added in reflected pairs or at \(q=1/2\); this freedom is \emph{symmetric slack}. These operations give every realizing law.}

\begin{redrevision}
\begin{proposition}[Realizability and all representatives]
\label{prop:signed-signature-realizability}
Let \(\mu\) be a finite signed Borel measure on \([0,1/4]\). There exists a probability law \(\Pi\) on \([0,1]\) with \(\omega_\Pi=\mu\) if and only if
\begin{equation}
\mu(\{1/4\})=0,
\qquad
c(\mu):=\int_{[0,1/4)}\frac{1}{\sqrt{1-4r}}\,|\mu|(\dd r)\le1.
\label{eq:signed-realizability}
\end{equation}
For admissible \(\mu=\mu^+-\mu^-\), lift \(\mu^+/\sqrt{1-4r}\) to \(q_+\) and \(\mu^-/\sqrt{1-4r}\) to \(q_-\), obtaining \(\Pi_{\min}\) of mass \(c(\mu)\). All realizing laws are exactly \(\Pi_{\min}+\sigma\), where \(\sigma\) is any nonnegative reflection-symmetric measure of mass \(1-c(\mu)\).
\end{proposition}
\end{redrevision}

\redrev{Specifying an admissible signature therefore specifies both a curve and its entire family of latent populations. Grouping voting-equivalent laws gives the \emph{voting quotient} \(\cP([0,1])/{\sim}\). Let \(\mathcal S\) be the signatures satisfying \eqref{eq:signed-realizability} and \(\mathcal V\) the attainable complete curves. The classification is the pair of bijections}
\begin{redrevision}
\[
\cP([0,1])/{\sim}
\ \longleftrightarrow\ \mathcal S
\ \longleftrightarrow\ \mathcal V,
\qquad
[\Pi]\longmapsto\omega_\Pi\longmapsto(V_n(\Pi))_{n\ge0}.
\]
\end{redrevision}
\redrev{Every attainable curve has exactly one signature, and the branch-lift construction gives all populations producing it. Appendix~\ref{app:binary-proofs} proves the mass constraint and the exhaustive construction; Theorem~\ref{thm:realizability} gives all density representatives.}

\subsection{\redrev{Curve Invariance and Identification Limits}}
\label{sec:curve-identification}

\redrev{The freedom left by a complete curve is visible even among smooth population densities.}
\begin{redrevision}
\begin{example}[Different populations, identical voting curves]
\label{ex:same-g}
Define
\[
f_1(q)=1+\tfrac14(2q-1),\qquad
f_2(q)=f_1(q)+4\bigl(q(1-q)-\tfrac16\bigr).
\]
Both are valid densities. Their difference is reflection-symmetric with integral zero, and both have signature density \(g(r)=\tfrac12\sqrt{1-4r}\). They are distinct populations with the same entire voting curve.
\end{example}
\end{redrevision}

\redrev{A complete curve identifies a single population precisely when the minimum construction uses all available probability mass. The following criterion expresses this \emph{saturation} directly in terms of a law and its reflection.}

\begin{redrevision}
\begin{corollary}
\label{cor:saturation}
For every latent law \(\Pi\),
\begin{equation}
c(\omega_\Pi)=d_{\mathrm{TV}}(\Pi,\iota_\#\Pi),\qquad
d_{\mathrm{TV}}(P,Q):=\tfrac12|P-Q|([0,1]).
\label{eq:saturation-tv}
\end{equation}
The complete odd-budget curve uniquely determines \(\Pi\) if and only if \(c(\omega_\Pi)=1\), equivalently \(\Pi\perp\iota_\#\Pi\).
\end{corollary}
\end{redrevision}

\redrev{Mutual singularity means that \(\Pi\) and its reflection have no common probability mass, so there is no symmetric component to redistribute. Saturation can hold with items on both sides of chance. For example, writing \(\delta_q\) for a unit atom at \(q\),}
\begin{redrevision}
\[
\Pi_* =\tfrac35\delta_{4/5}+\tfrac25\delta_{3/10},\qquad
\omega_{\Pi_*}=\tfrac9{25}\delta_{4/25}-\tfrac4{25}\delta_{21/100}.
\]
\end{redrevision}
\redrev{Its cost is \(c(\omega_{\Pi_*})=3/5+2/5=1\), so its curve uniquely identifies the population. In Example~\ref{ex:same-g}, the cost is instead \(1/8\). The remaining \(7/8\) of probability mass can be redistributed symmetrically, producing different populations with the same curve.}

\section{\redrev{Direction Changes and Total Variation}}
\label{sec:voting-shapes}

\redrev{By \eqref{eq:delta-omega}, each step's direction is the sign of a power-weighted average of the signature. A sign condition at every scale gives a sufficient criterion for monotonicity, together with diminishing gains or losses.}

\begin{proposition}
\label{prop:shape-criterion}
If \(\omega_\Pi\) is a nonnegative measure, then \(n\mapsto V_n(\Pi)\) is nondecreasing and discretely concave.  If \(-\omega_\Pi\) is nonnegative, then it is nonincreasing and discretely convex.
\end{proposition}

\redrev{For a population density \(f\), \(\omega_\Pi\ge0\) is equivalent to \(f(q_+(r))\ge f(q_-(r))\) almost everywhere. This \emph{reflected dominance} permits systematically wrong items: the upper branch need only outweigh its reflected lower branch at each scale. Both densities in Example~\ref{ex:same-g} satisfy this condition while assigning positive mass below \(1/2\).}

\redrev{Opposite signs at different scales can instead produce a reversal. Consider a realizable signature \(\omega=A\delta_{r_1}-B\delta_{r_2}\), where \(A,B>0\) and \(0<r_1<r_2<1/4\). Its normalized increment is}
\begin{redrevision}
\[
\frac{V_k-V_{k-1}}{\binom{2k-1}{k}}
=A r_1^k-B r_2^k
=B r_2^k\left\{\frac AB\left(\frac{r_1}{r_2}\right)^k-1\right\}.
\]
\end{redrevision}
\redrev{The ratio inside braces decreases with \(k\). If the first increment is positive, the negative mass at the larger scale eventually dominates and the curve turns downward. Reversing the signs produces a dip followed by improvement. These are the mechanisms behind the two reversing curves in Figure~\ref{fig:shape-examples}. With finitely many nonzero scales, the largest carrying nonzero signed mass eventually fixes the direction. An unrestricted population can contain infinitely many competing scales approaching the threshold.}

\begin{example}[Infinitely many direction changes]
\label{ex:infinite-oscillation}
Let \(r_j=\frac14e^{-2^{-j}}\uparrow1/4\), for \(j\ge1\). Put an atom at \(q_+(r_j)\) for even \(j\) and at \(q_-(r_j)\) for odd \(j\). With \(S=\sum_{j\ge1}e^{-j^2}/\sqrt{1-4r_j}<\infty\), the weights and resulting signature are
\[
a_j=\frac{e^{-j^2}}{S\sqrt{1-4r_j}},
\qquad
\omega_{\Pi_{\mathrm{osc}}}
=\frac1S\sum_{j\ge1}(-1)^j e^{-j^2}\delta_{r_j}.
\]
For \(k_j=3j2^j\), the moments \(\int r^{k_j}\omega_{\Pi_{\mathrm{osc}}}(\dd r)\) have sign \((-1)^j\) for all sufficiently large \(j\). By Theorem~\ref{thm:omega-moments}, the corresponding vote-curve increments alternate in sign, so the curve changes direction infinitely often.
\end{example}

\redrev{At each selected index \(k_j\), the construction makes the \(j\)th scale dominate the moment; successive dominant scales have opposite signs. Appendix~\ref{app:branch-proofs} bounds the other terms to establish the alternation. Each scale uses only one branch, so \(c(\omega_{\Pi_{\mathrm{osc}}})=1\): even an infinitely reversing curve can have a unique latent representative.}

\redrev{Every voting curve nevertheless has finite total variation. Writing \(\|\omega_\Pi\|_{\mathrm{TV}}\) for the signature's total variation norm,}
\begin{redrevision}
\begin{equation}
\sum_{n=0}^{\infty}|V_{n+1}(\Pi)-V_n(\Pi)|
\le \frac12\bigl\{c(\omega_\Pi)-\|\omega_\Pi\|_{\mathrm{TV}}\bigr\}
\le \frac12.
\label{eq:total-variation}
\end{equation}
\end{redrevision}
\redrev{Thus the number of turns can be infinite while the accumulated change remains bounded. The same probability-mass cost that controls realizability bounds this total change. Summing the increment kernels proves \eqref{eq:total-variation} (Appendix~\ref{app:branch-proofs}); Section~\ref{sec:variation-shape} studies how the location of the remaining imbalance controls convergence to the limit.}

\redrev{The same telescoping argument bounds the change between any two budgets without a support condition.}
\begin{redrevision}
\begin{proposition}
\label{prop:universal-variation}
For integers \(m>n\ge0\),
\begin{equation}
|V_m(\Pi)-V_n(\Pi)|
\le
\min\!\left\{1,\|\omega_\Pi\|_{\mathrm{TV}}
\sum_{k=n}^{m-1}\binom{2k+1}{k+1}4^{-(k+1)},
\frac{\|\omega_\Pi\|_{\mathrm{TV}}}{\sqrt{\pi}}
(\sqrt m-\sqrt n)\right\}.
\label{eq:variation-bound-sum}
\end{equation}
\end{proposition}
\end{redrevision}
\redrev{A curve can keep moving for a long time, but only at rate \(\sqrt m-\sqrt n\). Corollary~\ref{cor:near-zero} in Section~\ref{sec:variation-shape} replaces \(4^{-(k+1)}\) by \(a^{k+1}\) when the signature is supported away from the threshold scale, making the bound geometric.}

\section{\redrev{Finite-Repeat Information Boundary}}
\label{sec:finite-repeat-access}

\redrev{A complete curve is an idealized information source. A labeled evaluation instead observes \(J\) repeats on each of \(N\) iid items, summarized by correctness counts}
\[
C_i=\sum_{j=1}^{J}B_{ij},\qquad C_i\mid Q_i\sim\Bin(J,Q_i).
\]
\redrev{Here \(J\) controls what can be identified even from the population count law, whereas \(N\) controls estimation precision. Write \(a_\ell=\E Q^\ell\), with \(a_0=1\), and denote the signature moments by}
\[
s_k=\int r^k\,\omega_\Pi(\dd r)
=\E\!\left[(2Q-1)\{Q(1-Q)\}^k\right].
\]
\redrev{Each count probability is the expectation of a degree-\(J\) binomial polynomial. Conversely, factorial moments recover every raw moment through degree \(J\) \citep[Section~2.1]{tian2017}. Thus the grouped-count law and these raw moments contain exactly the same population information.}

\begin{redrevision}
\begin{theorem}
\label{thm:grouped-repeat-identification}
For fixed depth $J$, the count law identifies raw moments with unbiased estimates from $N$ iid items:
\[
 a_\ell=\E\!\left[\frac{(C_i)_\ell}{(J)_\ell}\right],\qquad
 \widehat a_\ell=\frac1N\sum_{i=1}^N\frac{(C_i)_\ell}{(J)_\ell},\qquad 0\le\ell\le J,
\]
where $(x)_\ell=x(x-1)\cdots(x-\ell+1)$ and $(x)_0=1$. It therefore identifies
\begin{equation}
    s_k=
    \sum_{\ell=0}^{k}(-1)^\ell \binom{k}{\ell}
    \{2a_{k+\ell+1}-a_{k+\ell}\},
    \qquad 0\le k\le \left\lfloor\frac{J-1}{2}\right\rfloor.
    \label{eq:sk-from-raw-moments}
\end{equation}
Substituting $\widehat a_\ell$ yields unbiased estimates of the identified $s_k$. Each fixed identified prefix is jointly asymptotically normal as $N\to\infty$; Appendix~\ref{app:finite-repeat-proofs} gives the covariance and full limit statement.
\end{theorem}
\end{redrevision}

\redrev{The degree \(2k+1\) of \((2q-1)\{q(1-q)\}^k\) gives the cutoff, and \(V_k-V_{k-1}=\binom{2k-1}{k}s_k\) converts identified moments to budget increments. Thus the count law identifies every odd-budget accuracy with at most \(J\) calls. This prefix is the limit of guaranteed point identification over the unrestricted model: each higher signature moment can differ between populations with the same count law. This also limits nonparametric maximum likelihood estimation (NPMLE), whose likelihood depends only on the count probabilities.}

\begin{redrevision}
\begin{proposition}
\label{prop:fixed-depth-nonidentification}
For every fixed $J\ge1$ and every integer $k>\lfloor(J-1)/2\rfloor$, there exist two latent laws $\Pi_1$ and $\Pi_2$ on $[0,1]$ that induce the same distribution of $C\sim\Bin(J,Q)$ but have different $k$th signed moments:
\[
    \int r^k\,\omega_{\Pi_1}(\dd r)
    \ne
    \int r^k\,\omega_{\Pi_2}(\dd r).
\]
\rev{Thus fixed-depth NPMLE cannot point-identify the full signature uniformly over the unrestricted model class, even at the population optimum.}
\end{proposition}
\end{redrevision}

\redrev{The proof preserves raw moments through degree \(J\) while changing the chosen higher signature moment. More items reduce sampling error but cannot distinguish these populations at the same repeat depth. Special boundary count laws can identify more than the guaranteed prefix.}

\redrev{Grouped counts and majority accuracies retain different information. A finite count law can retain reflection-symmetric features invisible to the entire voting curve, and all count depths together identify \(\Pi\). The complete voting curve identifies its equivalence class, which is a singleton at saturation.}

\redrev{When both \(N\) and \(J\) grow, a plug-in estimator recovers any fixed Lipschitz signature functional, with separate \(N^{-1/2}\) and \(J^{-1/2}\) terms in its mean absolute error bound (Proposition~\ref{prop:plugin-omega}). Point identification is also distinct from settling a particular comparison: polynomial approximation can bound an unobserved signature functional from a finite, imprecise curve. Proposition~\ref{prop:finite-curve-stability} separates approximation error from amplification of observation error. Both statements and the finite-repeat proofs are in Appendix~\ref{app:finite-repeat-proofs}.}

\section{\redrev{Endpoint Behavior and Convergence}}
\label{sec:variation-shape}

\redrev{Later budgets put relatively more weight on items near chance. This boundary behavior determines both the limiting accuracy and how well a finite budget approximates it. The itemwise majority limit and dominated convergence give, for every latent law,}
\[
V_\infty(\Pi):=\lim_{n\to\infty}V_n(\Pi)
=\Pi((1/2,1])+\frac12\Pi(\{1/2\}).
\]
\redrev{Although this expression uses the latent law, the endpoint depends only on the signature.}

\begin{proposition}
\label{prop:endpoint-omega}
For every latent law \(\Pi\),
\[
V_\infty(\Pi)
=\frac12+\frac12\int_{[0,1/4)}
\frac{1}{\sqrt{1-4r}}\,\omega_\Pi(\dd r),
\]
where the signed integral is absolutely integrable.  Hence the complete odd-budget curve determines its endpoint. \redrev{No fixed finite set of budgets identifies the endpoint over the unrestricted class.}
\end{proposition}

\redrev{The endpoint uses the same singular weight as the realizability cost \(c\), now applied to signed rather than absolute mass. Since \(\sqrt{1-4q(1-q)}=|2q-1|\), the weight removes the orientation magnitude and leaves upper-branch probability minus lower-branch probability. The mass constraint ensures integrability. Finite budgets instead test powers of \(r\), which give little separation between small imbalances close to the threshold.}

\redrev{For \(0<\varepsilon<1/2\), consider \(\Pi_\varepsilon^\pm=\delta_{1/2\pm\varepsilon}\). Their signatures and endpoints are}
\begin{redrevision}
\[
\omega_{\Pi_\varepsilon^\pm}
=\pm2\varepsilon\,\delta_{1/4-\varepsilon^2},
\qquad
V_\infty(\Pi_\varepsilon^+)=1,\quad
V_\infty(\Pi_\varepsilon^-)=0.
\]
\end{redrevision}
\redrev{Both laws are saturated, yet their signatures tend to zero in total variation and their accuracies approach \(1/2\) at every fixed finite budget, while their endpoints remain one unit apart. Unique realization therefore does not imply uniform endpoint stability in this norm. A support gap for the uncancelled imbalance supplies a direct bound on every remaining budget change.}

\begin{redrevision}
\begin{corollary}
\label{cor:near-zero}
If \(\supp(\omega_\Pi)\subseteq[0,a]\) for some \(a<1/4\), then for every \(m>n\ge0\),
\begin{equation}
|V_m(\Pi)-V_n(\Pi)|
\le
\|\omega_\Pi\|_{\mathrm{TV}}
\sum_{k=n}^{m-1}\binom{2k+1}{k+1}a^{k+1}
\le
\frac{2a}{1-4a}\|\omega_\Pi\|_{\mathrm{TV}}(4a)^n.
\label{eq:near-zero-bound}
\end{equation}
Letting \(m\to\infty\) gives the same exponential bound for
\(|V_\infty(\Pi)-V_n(\Pi)|\).
\end{corollary}
\end{redrevision}

\redrev{The condition concerns the \emph{uncancelled} imbalance. Symmetric mass can lie arbitrarily close to \(q=1/2\) without affecting the curve or slowing its convergence. Bounding each moment by \(a^{k+1}\|\omega_\Pi\|_{\mathrm{TV}}\) and summing a geometric series gives the result.}

\redrev{A different condition controls a representative population: the amount of probability mass near chance. It yields a power-law rate even when the signature has no support gap.}

\begin{theorem}
\label{thm:threshold-bridge}
Let \(M=2n+1\).  Suppose that for constants \(C,\kappa>0\) and \(t_0\in(0,1/2]\),
\begin{equation}
\Pi\!\left(\left|Q-\frac12\right|\le t\right)\le Ct^\kappa,
\qquad 0<t<t_0.
\label{eq:threshold-margin}
\end{equation}
Then
\begin{equation}
|V_n(\Pi)-V_\infty(\Pi)|
\le e^{-2Mt_0^2}
+C\Gamma\!\left(1+\frac{\kappa}{2}\right)(2M)^{-\kappa/2}.
\label{eq:threshold-bridge}
\end{equation}
\end{theorem}

\redrev{The exponent \(-\kappa/2\) is worst-case sharp over the stated local-margin class (Theorem~\ref{thm:rate-sharp}). The identity \(1/4-R=(Q-1/2)^2\) explains the rate: an item's probability of voting against its favored answer is at most \(e^{-2M(q-1/2)^2}\). Items farther than order \(M^{-1/2}\) from chance stabilize quickly; integrating the mass still near the threshold gives the power law. The matching lower bound places mass on one side at this scale.}

\redrev{Two special cases make the rate explicit. If \(\Pi(|Q-1/2|\le\delta)=0\) for some \(\delta>0\), only the first term remains and}
\begin{redrevision}
\begin{equation}
|V_n(\Pi)-V_\infty(\Pi)|\le e^{-2M\delta^2};
\label{eq:gap-bridge}
\end{equation}
\end{redrevision}
\redrev{if instead \(Q\) has a density bounded by \(B\), then \eqref{eq:threshold-margin} holds with \(C=2B\) and \(\kappa=1\), and}
\begin{redrevision}
\begin{equation}
|V_n(\Pi)-V_\infty(\Pi)|
\le e^{-M/2}+B\sqrt{\frac{\pi}{2M}}.
\label{eq:density-bridge}
\end{equation}
\end{redrevision}
\redrev{A population separated from chance converges exponentially; mass spread up to the threshold is resolved only at rate \(M^{-1/2}\), and by Theorem~\ref{thm:rate-sharp} this exponent cannot be improved over the corresponding margin class.}

\redrev{These two rate conditions control different objects. The margin condition constrains a representative law and can change under symmetric redistribution. Corollary~\ref{cor:near-zero} constrains the signature shared by the whole equivalence class. The proofs, including those of the two special cases, are in Appendix~\ref{app:branch-proofs}.}

\subsection{\redrev{Why the Binary Case Is Special}}
\label{sec:plurality-main}

\redrev{Binary voting is one-dimensional because every wrong vote is a vote for the only alternative. With \(K\ge3\) options, index the correct one by \(0\), let \(p\in\Delta_K\) be an item's per-call option probabilities, and let \(A_{m,K}(p)\) be the accuracy of the plurality of \(m\) calls under uniform tie-breaking (Appendix~\ref{app:multiclass-extension}). The correct-answer probability \(p_0\) then no longer determines the curve.}
\begin{redrevision}
\begin{theorem}
\label{thm:q-not-enough}
For every \(q\in(0,1)\) and \(K\ge3\), there exist two latent categorical vectors with the same correct-answer probability \(p_0=q\) but different three-vote plurality accuracies. For \(K=3\), the vectors \(p^{\mathrm{conc}}=(q,1-q,0)\) and \(p^{\mathrm{diff}}=(q,\tfrac{1-q}{2},\tfrac{1-q}{2})\) give
\begin{align}
    A_{3,3}(p^{\mathrm{conc}})&=q^3+3q^2(1-q),
    \label{eq:conc-three}\\
    A_{3,3}(p^{\mathrm{diff}})&=q^3+3q^2(1-q)+\frac{q(1-q)^2}{2}.
    \label{eq:diff-three}
\end{align}
For \(q\in(1/3,1/2)\) the two vectors also have different limits: the plurality accuracy of \(p^{\mathrm{conc}}\) tends to zero and that of \(p^{\mathrm{diff}}\) to one. For \(K>3\), append zero coordinates to both vectors.
\end{theorem}
\end{redrevision}
\redrev{Concentrated errors compete with the correct answer, whereas errors spread over several wrong options cancel in a plurality. A signature indexed by a scalar success probability therefore cannot classify plurality curves: the allocation of errors among wrong labels is part of the latent geometry. Under symmetric errors, \(p_1=\cdots=p_{K-1}=(1-q)/(K-1)\), the curve again depends only on the law of \(q\), with threshold \(1/K\) (Appendix~\ref{app:multiclass-extension}).}

\section{\redrev{Discussion}}
\label{sec:discussion}

\redrev{Binary majority reweights a population as the budget grows: items far from chance settle early, while unresolved imbalances closer to chance determine later changes. The signed signature captures exactly these imbalances. Its probability-mass constraint classifies all attainable curves and all populations producing them, with saturation marking the boundary at which an entire curve identifies a single law. This explains why one-call accuracy, an early gain, or even a fully known curve answers different questions about the population.}

\redrev{For finite evaluation, repeat depth determines the guaranteed identifiable coordinates, while uncertainty bounds can resolve selected comparisons beyond point identification. Endpoint extrapolation additionally requires control near chance. These distinctions suggest allocating repeats to the budget comparisons or endpoint functionals of interest. The classification is exact for binary unweighted majority under infinite exchangeability. Adaptive or weighted aggregation introduces additional structure. For multiple-choice plurality, the allocation of errors among wrong labels changes both finite-budget accuracy and the endpoint (Theorem~\ref{thm:q-not-enough}). An extension therefore needs to represent this allocation alongside correctness probability.}

\newpage
\appendix

\section{Additional structural statements}
\label{app:additional-structure}

\redrev{This appendix records the repeated-call representation, the even-budget identity, all density representatives of a voting-equivalence class, and the sharpness of the local-margin rate.}

\begin{proposition}
\label{prop:exchangeable-representation}
If \(B_1,B_2,\ldots\) is infinitely exchangeable, then there exists \(Q\sim\Pi\) such that, conditional on \(Q\), the \(B_j\) are iid \(\Bern(Q)\).  Consequently every such repeated-correctness experiment has population curve \(V_n(\Pi)=\int P_n(q)\Pi(\dd q)\).
\end{proposition}

\begin{proposition}
\label{prop:even-collapse}
For \(n\ge1\), define
\[
P^{\mathrm{even}}_n(q)
=\PP\{\Bin(2n,q)>n\}+\frac12\PP\{\Bin(2n,q)=n\}.
\]
Then \(P^{\mathrm{even}}_n(q)=P_{n-1}(q)\) for every \(q\in[0,1]\).
\end{proposition}

\redrev{Proposition~\ref{prop:signed-signature-realizability} characterizes admissible signatures and all their latent representatives. Its proof is in Appendix~\ref{app:binary-proofs}. The density formulation below gives explicit coordinates for the same reflection-symmetric freedom.}

Assume \(Q\) has density \(f\) on \([0,1]\), and define
\[
g(r)=f(q_+(r))-f(q_-(r)),
\qquad
s(r)=f(q_+(r))+f(q_-(r)).
\]
Here \(g\) is the branch-asymmetry density and \(s\) the branch-symmetric part.

\redrev{Theorem~\ref{thm:density-rep} in Section~\ref{sec:density-admissibility} gives the level and increment formulas in these coordinates; its proof is in Appendix~\ref{app:branch-proofs}.}

\begin{theorem}[\redrev{Density representatives of a voting-equivalence class}]
\label{thm:realizability}
Let \(g:[0,1/4]\to\R\) be measurable.  There exists a probability density \(f\) on \([0,1]\) with branch asymmetry \(g\) if and only if
\begin{equation}
\int_0^{1/4}\frac{|g(r)|}{\sqrt{1-4r}}\,\dd r\le1.
\label{eq:realizability}
\end{equation}
When this holds, all realizing densities are obtained by choosing measurable \(h\ge0\) satisfying
\begin{equation}
\int_0^{1/4}\frac{h(r)}{\sqrt{1-4r}}\,\dd r
=1-\int_0^{1/4}\frac{|g(r)|}{\sqrt{1-4r}}\,\dd r,
\label{eq:slack-h}
\end{equation}
and setting
\begin{equation}
f(q_+(r))=\frac{|g(r)|+h(r)+g(r)}{2},
\qquad
f(q_-(r))=\frac{|g(r)|+h(r)-g(r)}{2}
\label{eq:realizing-density}
\end{equation}
for Lebesgue-almost every \(r\in[0,1/4]\).  The realizing density is unique up to almost-everywhere equality if and only if equality holds in \eqref{eq:realizability}; in that saturated case \(f(q_+(r))=g^+(r)\) and \(f(q_-(r))=g^-(r)\) almost everywhere.
\end{theorem}

\redrev{The local-margin rate in Theorem~\ref{thm:threshold-bridge} is also worst-case sharp.}
\begin{theorem}
\label{thm:rate-sharp}
Fix \(C,\kappa>0\) and \(t_0\in(0,1/2]\), and let \(\cG(C,\kappa,t_0)\) be the class of laws satisfying \eqref{eq:threshold-margin}.  There exists \(c_*(C,\kappa,t_0)>0\) such that, for all sufficiently large odd \(M=2n+1\),
\[
\sup_{\Pi\in\cG(C,\kappa,t_0)}
|V_n(\Pi)-V_\infty(\Pi)|
\ge c_*(C,\kappa,t_0)M^{-\kappa/2}.
\]
Hence the exponent \(-\kappa/2\) in Theorem~\ref{thm:threshold-bridge} cannot be uniformly improved over this local-margin class.
\end{theorem}

\section{Latent laws used in Figure~\ref{fig:shape-examples}}
\label{app:figure1-laws}

\begin{redrevision}
All five laws below have mean $3/4$ and use ordinary majority with $2n+1$
conditionally iid calls. Write $\delta_q$ for a unit atom at $q$. For a
finite law $\Pi=\sum_j w_j\delta_{q_j}$, the increment identity gives
\[
s_k(\Pi):=\frac{V_k(\Pi)-V_{k-1}(\Pi)}{\binom{2k-1}{k}}
=\sum_j w_j(2q_j-1)\{q_j(1-q_j)\}^k,\qquad k\ge1.
\]
The itemwise majority limit gives
$V_\infty(\Pi)=\sum_{q_j>1/2}w_j+\frac12\sum_{q_j=1/2}w_j$.
An atom at $q=1$ changes the mean but contributes nothing to $s_k$.

The plotted laws are
\begin{align*}
\Pi_{\mathrm C}&=\tfrac14\delta_0+\tfrac34\delta_1
&&\text{(constant)},\\
\Pi_{\mathrm F}&=\tfrac13\delta_{1/4}+\tfrac23\delta_1
&&\text{(fast drop)},\\
\Pi_{\mathrm D}&=\tfrac5{24}\delta_{1/5}+\tfrac5{24}\delta_{3/5}
                 +\tfrac7{12}\delta_1
&&\text{(dip then surpass)},\\
\Pi_{\mathrm R}&=\tfrac3{10}\delta_{2/5}+\tfrac7{10}\delta_{9/10}
&&\text{(rise then fall)},\\
\Pi_{\mathrm S}&=\tfrac{25}{49}\delta_{51/100}+\tfrac{24}{49}\delta_1
&&\text{(slow rise)}.
\end{align*}
Their weights follow directly from the mean constraint. With support
$\{a,1\}$, it forces the weight at $a$ to be $1/[4(1-a)]$; choosing
$a=1/4$ or $a=51/100$ gives $\Pi_{\mathrm F}$ or $\Pi_{\mathrm S}$.
For $\Pi_{\mathrm D}$, choose equal masses at $a=1/2-3h$ and
$b=1/2+h$, and require $a(1-a)/[b(1-b)]=2/3$. This gives $h=1/10$,
hence $a=1/5$ and $b=3/5$. Their common weight $t$ then satisfies
$\mathbb E Q=1-t(4/5+2/5)$, so $t=5/24$. For $\Pi_{\mathrm R}$,
take $a=1/2-h$ and $b=1/2+4h$ and impose
$b(1-b)/[a(1-a)]=3/8$. Again $h=1/10$, giving $a=2/5$ and $b=9/10$;
then $(2/5)w+(9/10)(1-w)=3/4$ gives $w=3/10$. Thus every displayed
weight is positive, the weights sum to one, and every mean is exactly $3/4$.

The endpoint-supported law $\Pi_{\mathrm C}$ has $V_n=3/4$ for every $n$.
For $\Pi_{\mathrm F}$, $s_k=-(3/16)^k/6<0$, so the decrease is strict and
\[
V_\infty(\Pi_{\mathrm F})=\frac23,\qquad
0<V_n(\Pi_{\mathrm F})-\frac23\le\frac1{12}\left(\frac34\right)^n.
\]
Indeed, its positive decrements $d_k=V_{k-1}-V_k$ satisfy
$d_{k+1}/d_k=(3/4)(2k+1)/(2k+2)<3/4$. The residual is their tail sum,
hence $V_{n+1}-2/3\le(3/4)(V_n-2/3)$. This quantifies ``fast.''

\redrev{For the two reversing curves, the exact normalized increments factor as}
\begin{align*}
s_k(\Pi_{\mathrm D})
&=\frac1{24}\left(\frac6{25}\right)^k
  \left\{1-3\left(\frac23\right)^k\right\},\\
s_k(\Pi_{\mathrm R})
&=\frac3{50}\left(\frac6{25}\right)^k
  \left\{\frac{28}{3}\left(\frac38\right)^k-1\right\}.
\end{align*}
For the dip law, $3(2/3)^2=4/3>1$ and $3(2/3)^3=8/9<1$.
The geometric factor decreases strictly, so the increments are negative
for $k=1,2$ and positive for every $k\ge3$. For the rise--fall law,
$(28/3)(3/8)^2=21/16>1$ and $(28/3)(3/8)^3=63/128<1$;
its increments are positive for $k=1,2$ and negative for every $k\ge3$.
Thus the unique minimum and maximum, respectively, occur at $n=2$
(five calls); the direction changes on the step to seven calls. Moreover,
\[
V_\infty(\Pi_{\mathrm D})=\frac{19}{24}>\frac34,
\qquad
V_\infty(\Pi_{\mathrm R})=\frac7{10}<\frac34.
\]
Consequently the dip law eventually exceeds its starting accuracy, and
the rise--fall law eventually falls below it.

The larger nonzero signature scale is $6/25$ in both cases. Its positive
mass eventually dominates the negative mass at $4/25$ for
$\Pi_{\mathrm D}$; its negative mass eventually dominates the positive
mass at $9/100$ for $\Pi_{\mathrm R}$. The atom at $1$ in
$\Pi_{\mathrm D}$ sets the mean without affecting this competition.
In fact, two atoms cannot dip then rise at this mean. The only potentially
reversing case straddles $1/2$ and has signature
$A\delta_{r_+}-B\delta_{r_-}$ with $A-B=1/2$ and $A,B>0$.
If $s_1<0$, then $r_+/r_-<B/A<1$, forcing $s_k<0$ for all $k\ge1$.
The third atom is therefore essential for this role.

\redrev{For the slow-rise family, take $0<\varepsilon<1/4$ and set}
\[
\Pi_\varepsilon=w_\varepsilon\delta_{1/2+\varepsilon}
 +(1-w_\varepsilon)\delta_1,
\qquad w_\varepsilon=\frac1{2(1-2\varepsilon)}.
\]
This is the same mean-constrained construction above, and
$s_k(\Pi_\varepsilon)=\varepsilon(1/4-\varepsilon^2)^k/(1-2\varepsilon)>0$.
Hence $V_n(\Pi_\varepsilon)$ increases strictly to $1$. Its finite-budget
gain satisfies
\[
0\le V_n(\Pi_\varepsilon)-\frac34
\le\frac{\varepsilon}{2(1-2\varepsilon)}
       \left(\sqrt{2n+1}-1\right).
\]
To obtain the bound, use
$P_n'(q)=(2n+1)\binom{2n}{n}\{q(1-q)\}^n
\le(2n+1)\binom{2n}{n}4^{-n}\le\sqrt{2n+1}$
(the last inequality follows by induction), integrate from $1/2$ to
$1/2+\varepsilon$, and subtract $\varepsilon$.
Thus, as $\varepsilon\downarrow0$, the maximum gain through any horizon
$N$ with $\varepsilon\sqrt N\to0$ tends to zero, even though the limiting
gain is $1/4$ for every positive $\varepsilon$. The plotted law $\Pi_{\mathrm S}$ takes
$\varepsilon=1/100$.
\end{redrevision}

\section{Proofs for binary voting}
\label{app:binary-proofs}

\begin{proof}[Proof of Proposition~\ref{prop:exchangeable-representation}]
This is the binary de Finetti representation theorem.  Infinite exchangeability implies the existence of a mixing law $\Pi$ on $[0,1]$ such that, conditional on $Q\sim\Pi$, the sequence is iid Bernoulli$(Q)$.  Averaging the conditional majority accuracy $P_n(Q)$ over this mixing law gives the displayed curve formula.  Finite exchangeability alone gives an approximation or a finite-population representation, not an exact infinite mixing law; the paper uses the projective repeated-call idealization.
\end{proof}

\begin{proof}[Proof of Proposition~\ref{prop:even-collapse}]
Generate $2n$ iid correctness bits and then delete one of the $2n$ positions uniformly at random, independently of the bits.  The remaining $2n-1$ bits are again iid $\Bern(q)$, so their majority accuracy is $P_{n-1}(q)$.  Conditional on the original $2n$ bits, deleting one bit gives the same decision as the fair-tie even rule: a strict correct majority remains correct after deletion, a strict incorrect majority remains incorrect after deletion, and an exact $n$--$n$ tie becomes correct with probability $1/2$ because the deleted bit is equally likely to be correct or incorrect.  Hence the fair-tie $2n$-vote accuracy equals the $(2n-1)$-vote majority accuracy.
\end{proof}

\begin{proof}[Derivation of \eqref{eq:pointwise-diff} and fixed-item curvature]
Let $X\sim\Bin(2n+1,q)$ and $Y\sim\Bin(2,q)$ be independent.  Then $X+Y\sim\Bin(2n+3,q)$, and the two majority decisions differ only when $X=n$ or $X=n+1$.  If $X=n$, the old vote fails and the new vote succeeds exactly when $Y=2$.  If $X=n+1$, the old vote succeeds and the new vote fails exactly when $Y=0$.  Hence
\begin{align*}
P_{n+1}(q)-P_n(q)
&=\PP\{X=n\}\PP\{Y=2\}-\PP\{X=n+1\}\PP\{Y=0\}\\
&=\binom{2n+1}{n}q^{n+2}(1-q)^{n+1}
 -\binom{2n+1}{n+1}q^{n+1}(1-q)^{n+2}\\
&=\binom{2n+1}{n+1}q^{n+1}(1-q)^{n+1}(2q-1).
\end{align*}
\rev{For the shape statement, write $d_n(q)=P_{n+1}(q)-P_n(q)$. For $q\in(0,1)\setminus\{1/2\}$,}
\begin{equation*}
\frac{d_{n+1}(q)}{d_n(q)}
=\frac{\binom{2n+3}{n+2}}{\binom{2n+1}{n+1}}q(1-q)
=\frac{2(2n+3)}{n+2}q(1-q)
\le\frac{2n+3}{2n+4}<1.
\end{equation*}
\rev{On this domain the sign of $d_n(q)$ is the sign of $2q-1$, giving monotonicity and discrete curvature. At $q=0,1/2,1$ the curve is constant, which handles the cases where the ratio is undefined.}
\end{proof}

\begin{proof}[Proof of Theorem~\ref{thm:omega-moments}]
By definition,
\begin{equation*}
\omega_\Pi([0,1/4])=\int(2q-1)\,\Pi(\dd q)=2\E Q-1=2V_0-1,
\end{equation*}
which gives \eqref{eq:v0-omega}.  Averaging \eqref{eq:pointwise-diff} over $Q\sim\Pi$ gives
\begin{equation*}
V_{n+1}-V_n
=\binom{2n+1}{n+1}\int q^{n+1}(1-q)^{n+1}(2q-1)\,\Pi(\dd q).
\end{equation*}
This equals \eqref{eq:delta-omega} by the definition of the signed pushforward.

The curve determines the moments
\begin{equation*}
    m_0=\int\omega_\Pi(\dd r)=2V_0-1,
    \qquad
    m_k=\int r^k\omega_\Pi(\dd r)=\frac{V_k-V_{k-1}}{\binom{2k-1}{k}},\quad k\ge1.
\end{equation*}
If two finite signed measures on $[0,1/4]$ have the same moments of all orders, their difference integrates every polynomial to zero.  By Stone--Weierstrass, polynomials are uniformly dense in $C([0,1/4])$, so the difference integrates every continuous function to zero.  A finite signed Borel measure on a compact metric space is determined by its integrals against continuous functions; hence the two measures coincide.  This proves uniqueness and the stated equivalence between the full curve and the signed voting signature.

\begin{redrevision}
Finally, let \(H=\Pi_1-\Pi_2\). Every reflection-antisymmetric polynomial has the form \((2q-1)p(q(1-q))\): write it as an odd polynomial in \(2q-1\) and use \((2q-1)^2=1-4q(1-q)\). Equal signatures therefore make \(H\) annihilate all such polynomials. Antisymmetrizing polynomial approximations shows that \(H\) annihilates every continuous reflection-antisymmetric function. For any continuous \(f\), apply this to \(f-f\circ\iota\) to obtain \(\int f\dd H=\int f\dd(\iota_\#H)\); hence \(H=\iota_\#H\). Conversely, reflection symmetry cancels the orientation weight, giving equal signatures. Since \(H\) is a difference of probability laws, these invisible redistributions have total mass zero.
\end{redrevision}
\end{proof}

\begin{proof}[Proof of Proposition~\ref{prop:endpoint-omega}]
Let $T(q)=q(1-q)$.  By definition, $\omega_\Pi=T_\#((2q-1)\Pi)$, and hence
\begin{equation*}
    |\omega_\Pi|\le T_\#(|2q-1|\Pi)
\end{equation*}
as finite measures.  Therefore
\begin{equation*}
\int_{[0,1/4)}\frac{1}{\sqrt{1-4r}}\,|\omega_\Pi|(\dd r)
\le
\int_{[0,1]\setminus\{1/2\}}\frac{|2q-1|}{\sqrt{1-4q(1-q)}}\,\Pi(\dd q)
\le 1,
\end{equation*}
so the integral is absolutely well defined.  For $q\ne1/2$,
\begin{equation*}
\frac{2q-1}{\sqrt{1-4q(1-q)}}
=\begin{cases}
1,&q>1/2,\\
-1,&q<1/2.
\end{cases}
\end{equation*}
The signed pushforward gives
\begin{align*}
\int_{[0,1/4)}\frac{1}{\sqrt{1-4r}}\,\omega_\Pi(\dd r)
&=\Pi((1/2,1])-\Pi([0,1/2)).
\end{align*}
The right-hand side equals $2V_\infty(\Pi)-1$ under the convention that an atom at $1/2$ contributes half to $V_\infty$.

\begin{redrevision}
For any nonempty finite set of budgets, let $D$ be the largest degree of their majority polynomials. Choose $x_0<1/2$ and distinct $x_1,\ldots,x_{D+1}>1/2$ in $[0,1]$. The $(D+1)\times(D+2)$ moment matrix has a nonzero null vector $\alpha$ satisfying
\[
\sum_{i=0}^{D+1}\alpha_i x_i^\ell=0,\qquad \ell=0,\ldots,D.
\]
Moreover, $\alpha_0\ne0$: otherwise invertibility of the Vandermonde matrix on $x_1,\ldots,x_{D+1}$ would force $\alpha=0$. For sufficiently small $\varepsilon>0$, the laws
\[
\Pi_\pm=\sum_{i=0}^{D+1}\left(\frac1{D+2}\pm\varepsilon\alpha_i\right)\delta_{x_i}
\]
have positive weights summing to one and identical moments through degree $D$. Their accuracies therefore agree at every specified budget, while
$V_\infty(\Pi_+)-V_\infty(\Pi_-)=2\varepsilon\sum_{i=1}^{D+1}\alpha_i=-2\varepsilon\alpha_0\ne0$.
\end{redrevision}
\end{proof}

\begin{proof}[Proof of Proposition~\ref{prop:signed-signature-realizability}]
Necessity follows from the construction of $\omega_\Pi$. The point $r=1/4$ has only the preimage $q=1/2$, where the orientation $2q-1$ is zero, so $\omega_\Pi(\{1/4\})=0$. The total-variation domination argument in the proof of Proposition~\ref{prop:endpoint-omega} gives \eqref{eq:signed-realizability}.

\begin{redrevision}
For sufficiency, write \(\mu=\mu^+-\mu^-\) and define
\[
\Pi_{\min}(A)=
\int_{[0,1/4)}\frac{\1\{q_+(r)\in A\}}{\sqrt{1-4r}}\,\mu^+(\dd r)
+\int_{[0,1/4)}\frac{\1\{q_-(r)\in A\}}{\sqrt{1-4r}}\,\mu^-(\dd r).
\]
Its mass is the weighted total variation in \eqref{eq:signed-realizability}. The two lifts undo the orientation factors \(\pm\sqrt{1-4r}\), so their signed pushforward is \(\mu^+-\mu^-=\mu\). Any nonnegative reflection-symmetric \(\sigma\) has zero signed pushforward. Choosing its mass to be \(1-\Pi_{\min}([0,1])\) therefore makes \(\Pi_{\min}+\sigma\) a probability law with signature \(\mu\); a threshold atom is one such choice.

To show that these are all realizations, let \(\Pi\) realize \(\mu\). Put \(T(q)=q(1-q)\) and
\[
M_+=T_\#\bigl(\Pi|_{(1/2,1]}\bigr),\qquad
M_-=T_\#\bigl(\Pi|_{[0,1/2)}\bigr),\qquad
\nu^\pm(\dd r)=\frac{\mu^\pm(\dd r)}{\sqrt{1-4r}}.
\]
Then \(M_+-M_-=\nu^+-\nu^-\). Since \(\nu^+\) and \(\nu^-\) are mutually singular, the minimality of the Jordan decomposition gives a nonnegative measure \(\lambda\) with
\[
M_+=\nu^++\lambda,\qquad M_-=\nu^-+\lambda.
\]
Lifting back along the two branches yields
\[
\Pi=\Pi_{\min}+\sigma,\qquad
\sigma=(q_+)_\#\lambda+(q_-)_\#\lambda+\Pi(\{1/2\})\delta_{1/2}.
\]
This \(\sigma\) is nonnegative, reflection-symmetric, and has the required remaining mass. This proves exhaustiveness and the minimal-mass interpretation. In particular, saturation of \eqref{eq:signed-realizability} leaves a unique latent representative; positive slack admits different symmetric completions.
\end{redrevision}
\end{proof}

\begin{redrevision}
\begin{proof}[Proof of Corollary~\ref{cor:saturation}]
Use the branch measures \(M_+,M_-\) from the preceding proof and write \(D=M_+-M_-\) on \([0,1/4)\). Since \(\omega_\Pi(\dd r)=\sqrt{1-4r}\,D(\dd r)\) and the weight is strictly positive there,
\[
c(\omega_\Pi)=|D|([0,1/4)).
\]
Reflection exchanges the branches and fixes the central atom, so
\[
H:=\Pi-\iota_\#\Pi=(q_+)_\#D-(q_-)_\#D.
\]
The two branch ranges are disjoint Borel sets, and each branch map is a Borel bijection. Total variation is therefore preserved on each branch, giving
\[
\|H\|_{\mathrm{TV}}=2|D|([0,1/4))=2c(\omega_\Pi).
\]
This proves \eqref{eq:saturation-tv}, including its factor of two. The Jordan positive part is \(H^+=(q_+)_\#D^++(q_-)_\#D^-=\Pi_{\min}\). Consequently the symmetric slack in any representative is exactly
\[
\sigma=\Pi-H^+=\Pi\wedge\iota_\#\Pi.
\]
Here the common part \(P\wedge Q\) has density \(\min\{p,q\}\) with respect to \(P+Q\), where \(p,q\) are the two Radon--Nikodym densities. Its mass is \(1-c(\omega_\Pi)\), and it vanishes exactly when \(\Pi\perp\iota_\#\Pi\). When \(c=1\), Proposition~\ref{prop:signed-signature-realizability} leaves only \(\Pi_{\min}\). When \(c<1\), the distinct completions
\[
\Pi_{\min}+(1-c)\delta_{1/2},\qquad
\Pi_{\min}+\tfrac{1-c}{2}(\delta_0+\delta_1)
\]
have the same signature and curve. This proves both directions of the uniqueness criterion.

For a density \(f\), the criterion reads
\[
c(\omega_\Pi)=\frac12\int_0^1|f(q)-f(1-q)|\dd q
=1-\int_0^1\min\{f(q),f(1-q)\}\dd q.
\]
Thus saturation is equivalent to \(f(q)f(1-q)=0\) almost everywhere: the two branches have no common probability mass. It need not put every item on one side of the threshold.

For the main-text example \(\Pi_*\), the branch scales are \(4/25\) and \(21/100\), and the orientation magnitudes are \(3/5\) and \(2/5\). Its signature therefore has cost \((9/25)/(3/5)+(4/25)/(2/5)=1\), as claimed.
\end{proof}
\end{redrevision}

\section{Proofs for branch asymmetry, variation, and shape}
\label{app:branch-proofs}

\begin{proof}[Proof of Theorem~\ref{thm:density-rep}]
On the upper branch $q=q_+(r)$,
\begin{equation*}
2q_+(r)-1=\sqrt{1-4r},
\qquad
\left|\frac{\dd q_+}{\dd r}\right|=\frac{1}{\sqrt{1-4r}}.
\end{equation*}
Thus the upper branch contributes $f(q_+(r))\dd r$ to $\omega_\Pi$.  On the lower branch $q=q_-(r)$,
\begin{equation*}
2q_-(r)-1=-\sqrt{1-4r},
\qquad
\left|\frac{\dd q_-}{\dd r}\right|=\frac{1}{\sqrt{1-4r}},
\end{equation*}
so the lower branch contributes $-f(q_-(r))\dd r$.  Hence $\omega_\Pi(\dd r)=g(r)\dd r$.  The increment formula follows from Theorem~\ref{thm:omega-moments}.

For the level formula, split the integral defining $V_n$ over the two branches:
\begin{align*}
V_n
&=\int_0^{1/4}\frac{P_n(q_+(r))f(q_+(r))+P_n(q_-(r))f(q_-(r))}{\sqrt{1-4r}}\,\dd r.
\end{align*}
Since $P_n(1-q)=1-P_n(q)$, this equals
\begin{align*}
&\frac12\int_0^{1/4}\frac{f(q_+(r))+f(q_-(r))}{\sqrt{1-4r}}\,\dd r \\
&\quad +\frac12\int_0^{1/4}
\frac{2P_n(q_+(r))-1}{\sqrt{1-4r}}
\{f(q_+(r))-f(q_-(r))\}\,\dd r.
\end{align*}
The first term is $1/2$ because $f$ integrates to one.  To evaluate the endpoint of $K_n$, write $x=q_+(r)-1/2=\sqrt{1-4r}/2$.  Since $P_n(1/2)=1/2$ and $P_n$ is a polynomial,
\[
    \frac{2P_n(q_+(r))-1}{\sqrt{1-4r}}
    =\frac{P_n(1/2+x)-P_n(1/2)}{x}
    \longrightarrow P_n'(1/2).
\]
The derivative identity follows from differentiating the binomial tail, or equivalently from the beta-density identity for the median tail,
\[
    P_n'(q)=(2n+1)\binom{2n}{n}q^n(1-q)^n.
\]
Thus the singularity at $r=1/4$ is removable, and the second term is \eqref{eq:g-level}.
\end{proof}

\begin{proof}[Proof of Theorem~\ref{thm:realizability}]
If $f$ realizes $g$, set $s(r)=f(q_+(r))+f(q_-(r))$.  Nonnegativity of the two branch values implies $s(r)\ge |g(r)|$ almost everywhere, and the change of variables gives
\begin{equation*}
1=\int_0^1f(q)\,\dd q
=\int_0^{1/4}\frac{s(r)}{\sqrt{1-4r}}\,\dd r.
\end{equation*}
Therefore \eqref{eq:realizability} is necessary.

Conversely, suppose \eqref{eq:realizability} holds and choose $h\ge0$ satisfying \eqref{eq:slack-h}.  Define branch values by \eqref{eq:realizing-density}.  They are nonnegative, their difference is $g$, and their sum is $|g|+h$.  Hence
\begin{equation*}
\int_0^1f(q)\,\dd q
=\int_0^{1/4}\frac{|g(r)|+h(r)}{\sqrt{1-4r}}\,\dd r=1.
\end{equation*}
Thus $f$ is a density with branch asymmetry $g$.  The same argument shows that every realizing density arises from $h=s-|g|$.  If the inequality is strict, nonzero slack choices exist and uniqueness fails.  If equality holds, the slack integral is zero; since $h\ge0$, this forces $h=0$ almost everywhere, giving the stated saturated form.  The converse uniqueness implication follows immediately.
\end{proof}

\begin{proof}[Proof of Corollary~\ref{cor:moment-recovery}]
The identity for $k=0$ is \eqref{eq:v0-omega}.  For $k\ge1$, set $n=k-1$ in \eqref{eq:delta-omega} and rearrange.
\end{proof}

\begin{proof}[Proof of Proposition~\ref{prop:universal-variation}]
By telescoping and Theorem~\ref{thm:omega-moments},
\begin{align*}
|V_m-V_n|
&\le \sum_{k=n}^{m-1}|V_{k+1}-V_k| \\
&\le \|\omega_\Pi\|_{\mathrm{TV}}
    \sum_{k=n}^{m-1}\binom{2k+1}{k+1}\sup_{0\le r\le1/4} r^{k+1},
\end{align*}
which gives the untruncated sum bound.  For every $k\ge0$,
\begin{equation*}
\binom{2k+1}{k+1}4^{-(k+1)}
=\frac12\binom{2k+2}{k+1}4^{-(k+1)}
\le \frac{1}{2\sqrt{\pi(k+1)}},
\end{equation*}
using the standard central-binomial bound $\binom{2j}{j}4^{-j}\le1/\sqrt{\pi j}$.  Since
\begin{equation*}
\frac{1}{2\sqrt{k+1}}\le \sqrt{k+1}-\sqrt{k},
\end{equation*}
we obtain the telescoping estimate
\begin{equation*}
\sum_{k=n}^{m-1}\binom{2k+1}{k+1}4^{-(k+1)}
\le \frac{1}{\sqrt{\pi}}(\sqrt{m}-\sqrt{n}).
\end{equation*}
Since both vote accuracies lie in $[0,1]$, combining these bounds with the bound $1$ gives \eqref{eq:variation-bound-sum}.

\begin{redrevision}
For the total-variation bound \eqref{eq:total-variation}, use
\[
\sum_{n=0}^{\infty}\binom{2n+1}{n+1}r^{n+1}
=\frac12\bigl\{(1-4r)^{-1/2}-1\bigr\},\qquad 0\le r<1/4,
\]
which follows from $\binom{2n+1}{n+1}=\frac12\binom{2n+2}{n+1}$ and the central-binomial generating function. The increment identity and Tonelli's theorem give
\begin{align*}
\sum_{n=0}^{\infty}|V_{n+1}-V_n|
&\le \int_{[0,1/4)}\sum_{n=0}^{\infty}\binom{2n+1}{n+1}r^{n+1}\,|\omega_\Pi|(\dd r)\\
&=\frac12\bigl\{c(\omega_\Pi)-\|\omega_\Pi\|_{\mathrm{TV}}\bigr\}\le\frac12.
\end{align*}
Proposition~\ref{prop:signed-signature-realizability} ensures that $\omega_\Pi(\{1/4\})=0$ and $c(\omega_\Pi)\le1$, so the integral is finite.
\end{redrevision}
\end{proof}

\begin{proof}[Proof of Corollary~\ref{cor:near-zero}]
Repeat the telescoping argument from Proposition~\ref{prop:universal-variation}, now using $r\le a$ on the support of $\omega_\Pi$.  The elementary bound $\binom{2k+1}{k+1}\le2\cdot4^k$ gives
\begin{equation*}
\binom{2k+1}{k+1}a^{k+1}\le 2a(4a)^k.
\end{equation*}
Summing the resulting geometric series yields \eqref{eq:near-zero-bound}.  Letting $m\to\infty$ gives the endpoint claim by dominated convergence of majority votes.
\end{proof}

\begin{proof}[Proof of Theorem~\ref{thm:threshold-bridge}]
For $q\ne1/2$, Hoeffding's inequality gives
\begin{equation*}
    |P_n(q)-\1\{q>1/2\}|
    \le \exp\{-2M(q-1/2)^2\},
    \qquad M=2n+1.
\end{equation*}
The margin condition implies $\Pi(\{1/2\})=0$ by letting $t\downarrow0$.  Therefore
\begin{equation*}
    |V_n(\Pi)-\Vinf(\Pi)|
    \le \E \exp\{-2M\Delta^2\},
    \qquad \Delta=|Q-1/2|.
\end{equation*}
Use the layer-cake identity with $S=\Delta^2$ and split at $t_0^2$:
\begin{equation*}
\E e^{-2MS}
\le e^{-2Mt_0^2}+\int_0^{t_0^2} 2M e^{-2Mt}\,\Pi(S\le t)\,\dd t.
\end{equation*}
The local margin condition implies $\Pi(S\le t)\le C t^{\kappa/2}$ for $0<t<t_0^2$, so
\begin{equation*}
\int_0^{t_0^2} 2M e^{-2Mt} C t^{\kappa/2}\,\dd t
\le C\Gamma\!\left(1+\frac{\kappa}{2}\right)(2M)^{-\kappa/2}.
\end{equation*}
This proves \eqref{eq:threshold-bridge}.  If $\Delta\ge\delta$ almost surely, the Hoeffding bound directly gives \eqref{eq:gap-bridge}.  If the density of $Q$ is bounded by $B$ on $[0,1]$, then $\Pi(|Q-1/2|\le t)\le2Bt$ for all $0<t<1/2$; applying the theorem with $C=2B$, $\kappa=1$, and $t_0=1/2$ gives \eqref{eq:density-bridge}.
\end{proof}

\begin{lemma}
\label{lem:local-binomial-tail}
Fix $c>0$.  For all sufficiently large odd $M$, let $q_M=1/2+c/\sqrt M$.  There exist constants $p_c>0$ and $M_c<\infty$ such that
\begin{equation*}
    \PP\left\{\Bin(M,q_M)\le\frac{M-1}{2}\right\}\ge p_c,
    \qquad M\ge M_c.
\end{equation*}
\end{lemma}

\begin{proof}
Let $X_M\sim\Bin(M,q_M)$.  The triangular-array central limit theorem gives
\begin{equation*}
    \frac{X_M-Mq_M}{\sqrt{Mq_M(1-q_M)}}\Rightarrow \mathcal N(0,1).
\end{equation*}
Moreover,
\begin{equation*}
    \frac{(M-1)/2-Mq_M}{\sqrt{Mq_M(1-q_M)}}\to -2c.
\end{equation*}
Hence the lower-tail probability converges to $\Phi(-2c)>0$.  Taking $p_c=\Phi(-2c)/2$ for all sufficiently large $M$ proves the claim.
\end{proof}

\begin{proof}[Proof of Theorem~\ref{thm:rate-sharp}]
Let $a=\min\{1/2, C t_0^\kappa/2\}$.  Let $X$ be supported on $(0,t_0]$ with distribution function $\PP(X\le t)=(t/t_0)^\kappa$ for $0<t\le t_0$, and define
\begin{equation*}
Q=\begin{cases}
1, & \text{with probability }1-a,\\
1/2+X, & \text{with probability }a.
\end{cases}
\end{equation*}
This law belongs to $\cG(C,\kappa,t_0)$ because for $0<t<t_0$,
\begin{equation*}
\PP(|Q-1/2|\le t)=a(t/t_0)^\kappa\le Ct^\kappa.
\end{equation*}
It is supported above $1/2$, so $\Vinf(\Pi)=1$.  Let
\begin{equation*}
q_M=\frac12+\frac{1}{4\sqrt M},
\qquad
p_M=\PP\left\{\Bin(M,q_M)\le\frac{M-1}{2}\right\}.
\end{equation*}
Lemma~\ref{lem:local-binomial-tail} with $c=1/4$ gives $p_M\ge p_0$ for all sufficiently large odd $M$ and some $p_0>0$.  For sufficiently large $M$, $1/(4\sqrt M)\le t_0$.  Since majority error is decreasing in $q$ on $[1/2,1]$,
\begin{align*}
1-V_n(\Pi)
&=\E[1-P_n(Q)]\\
&\ge a\,\PP\left(X\le\frac{1}{4\sqrt M}\right)\{1-P_n(q_M)\}\\
&=a\left(\frac{1}{4t_0\sqrt M}\right)^\kappa p_M
\ge a(4t_0)^{-\kappa}p_0 M^{-\kappa/2}.
\end{align*}
Because $\Vinf(\Pi)=1$, this proves the lower bound with $c_*(C,\kappa,t_0)=a(4t_0)^{-\kappa}p_0$.
\end{proof}

\begin{proof}[Proof of Proposition~\ref{prop:shape-criterion}]
Assume $\omega_\Pi$ is nonnegative.  Theorem~\ref{thm:omega-moments} gives
\(V_{n+1}-V_n\ge0\).  Moreover,
\begin{align*}
V_{n+2}-V_{n+1}
&=\binom{2n+3}{n+2}\int r^{n+2}\,\omega_\Pi(\dd r)\\
&=\int \left\{\frac{\binom{2n+3}{n+2}}{\binom{2n+1}{n+1}}r\right\}
\binom{2n+1}{n+1}r^{n+1}\,\omega_\Pi(\dd r).
\end{align*}
Since $0\le r\le1/4$,
\[
\frac{\binom{2n+3}{n+2}}{\binom{2n+1}{n+1}}r
=\frac{2(2n+3)}{n+2}r
\le\frac{2n+3}{2n+4}<1.
\]
Therefore \(V_{n+2}-V_{n+1}\le V_{n+1}-V_n\), proving discrete concavity.  The case $-\omega_\Pi\ge0$ follows by reversing signs.
\end{proof}

\begin{proof}[Proof of Example~\ref{ex:infinite-oscillation}]
The normalizing constant $S$ is finite because $1-4r_j=1-e^{-2^{-j}}\asymp2^{-j}$, so the summands are of order $e^{-j^2}2^{j/2}$.  By the definition of $S$, the coefficients $a_j$ sum to one, so the construction defines a probability measure.  If $j$ is even, the atom is placed on the upper branch and contributes $S^{-1}e^{-j^2}\delta_{r_j}$ to $\omega$; if $j$ is odd, it is placed on the lower branch and contributes $-S^{-1}e^{-j^2}\delta_{r_j}$.  This gives the displayed formula for $\omega_{\Pi_{\mathrm{osc}}}$. \redrev{The scales are distinct and \(c(\omega_{\Pi_{\mathrm{osc}}})=S^{-1}\sum_j e^{-j^2}/\sqrt{1-4r_j}=1\); Corollary~\ref{cor:saturation} therefore also proves the asserted uniqueness.}

It remains to check the signs of the selected moments.  Since $r_j=\frac14 e^{-2^{-j}}$,
\begin{equation*}
    4^k S\int r^k\,\omega_{\Pi_{\mathrm{osc}}}(\dd r)
    =\sum_{\ell=1}^{\infty}(-1)^\ell
      \exp\{-\ell^2-k2^{-\ell}\}.
\end{equation*}
Set $k_j=3j2^j$ and write
\begin{equation*}
    \phi_\ell^{(j)}=\ell^2+k_j2^{-\ell}.
\end{equation*}
The term $\ell=j$ is the unique dominant term.  Indeed, for $h\ge1$,
\begin{align*}
    \phi_{j+h}^{(j)}-\phi_j^{(j)}
    &=j(2h-3+3\cdot2^{-h})+h^2
      \ge \frac j2+h^2,\\
    \phi_{j-h}^{(j)}-\phi_j^{(j)}
    &=j(3\cdot2^h-3-2h)+h^2
      \ge j+h^2,
\end{align*}
where the second line is used only when $1\le h\le j-1$.  Hence the total contribution of all non-$j$ terms, divided by the magnitude of the $j$th term, is at most
\begin{equation*}
    \sum_{\ell\ne j}\exp\{-(\phi_\ell^{(j)}-\phi_j^{(j)})\}
    \le 2e^{-j/2}\sum_{h=1}^{\infty}e^{-h^2},
\end{equation*}
which is smaller than one for all sufficiently large $j$.  The sign of the selected moment is therefore the sign of the $j$th term, namely $(-1)^j$.  Since
\begin{equation*}
    V_{k_j}-V_{k_j-1}
    =\binom{2k_j-1}{k_j}\int r^{k_j}\,\omega_{\Pi_{\mathrm{osc}}}(\dd r),
\end{equation*}
the increments along this subsequence alternate in sign.  Between each pair of consecutive selected indices with opposite increment signs, the full increment sequence must change sign at least once.  Since the selected indices are unbounded, the curve changes direction infinitely many times.
\end{proof}

\section{\redrev{Finite-repeat details and proofs}}
\label{app:finite-repeat-proofs}
\label{app:finite-repeat-details}

\redrev{We use the grouped counts \(C_i\), raw moments \(a_\ell\), and signature moments \(s_k\) defined in Section~\ref{sec:finite-repeat-access}. The likelihood depends on the same finite set of count probabilities.}

The nonparametric MLE over mixing laws solves
\[
    \widehat\Pi\in\arg\max_{\Pi\in\cP([0,1])}
    \sum_{c=0}^{J}N_c\log\int_0^1 \binom{J}{c}q^c(1-q)^{J-c}\,\Pi(\dd q),
\]
where $N_c$ is the number of examples with count $c$.  A maximizer can be chosen with finite support.  At fixed $J$, latent laws with the same $J$-repeat count distribution have the same likelihood, and all population maximizers share the identified prefix in Theorem~\ref{thm:grouped-repeat-identification}.  \rev{Outside special identifiable count laws, a selected value for a higher signed moment or the full $\omega_\Pi$ reflects additional restrictions or regularization. Its feasible range, however, is already constrained by the observed count law.}

\begin{proof}[Proof of Theorem~\ref{thm:grouped-repeat-identification}]
Conditional on $Q=q$, the factorial moments of $C\sim\Bin(J,q)$ are
\begin{equation*}
    \E[(C)_\ell\mid Q=q]=(J)_\ell q^\ell,
    \qquad 0\le \ell\le J.
\end{equation*}
Averaging over $Q$ gives $a_\ell=\E[(C)_\ell/(J)_\ell]$.  Next expand
\begin{align*}
    s_k
    &=\E\bigl[(2Q-1)Q^k(1-Q)^k\bigr] \\
    &=\sum_{\ell=0}^{k}(-1)^\ell \binom{k}{\ell}
      \E\bigl[2Q^{k+\ell+1}-Q^{k+\ell}\bigr],
\end{align*}
which is \eqref{eq:sk-from-raw-moments}.  This requires moments only up to degree $2k+1$, so it is identified whenever $2k+1\le J$.

\begin{redrevision}
For fixed \(L\le\lfloor(J-1)/2\rfloor\), define
\[
z_{J,k}(c)=\sum_{\ell=0}^{k}(-1)^\ell\binom{k}{\ell}
\left\{2\frac{(c)_{k+\ell+1}}{(J)_{k+\ell+1}}
-\frac{(c)_{k+\ell}}{(J)_{k+\ell}}\right\},
\qquad 0\le k\le L.
\]
Then \(\widehat s_k=N^{-1}\sum_i z_{J,k}(C_i)\) is unbiased. The vector \(Z_i=(z_{J,0}(C_i),\ldots,z_{J,L}(C_i))\) is bounded and iid, so the multivariate central limit theorem gives
\[
\sqrt N\{(\widehat s_0,\ldots,\widehat s_L)-(s_0,\ldots,s_L)\}
\Rightarrow\mathcal N(0,\Sigma_{J,L}),\qquad
\Sigma_{J,L}=\operatorname{Cov}(Z_i).
\]
This is the joint Gaussian limit asserted in Theorem~\ref{thm:grouped-repeat-identification}.
\end{redrevision}
\end{proof}

\begin{proof}[Proof of Proposition~\ref{prop:fixed-depth-nonidentification}]
The distribution of $C\sim\Bin(J,Q)$ is determined by the raw moments $a_0,\ldots,a_J$, because each probability $\PP(C=c)$ is the expectation of a polynomial in $Q$ of degree at most $J$.

\redrev{Fix any integer $k>\lfloor(J-1)/2\rfloor$ and set}
\begin{equation*}
    h(q)=(2q-1)\{q(1-q)\}^k.
\end{equation*}
This polynomial has degree $2k+1>J$, so it is not in the span of $1,q,\ldots,q^J$.  Hence there exist distinct points $x_0,\ldots,x_{J+1}\in(0,1)$ and signed weights $\alpha_i$ such that
\begin{equation*}
    \sum_i \alpha_i x_i^\ell=0,\qquad \ell=0,\ldots,J,
    \quad\text{but}\quad
    \sum_i \alpha_i h(x_i)\ne0.
\end{equation*}
For example, take a nonzero vector in the nullspace of the $(J+1)\times(J+2)$ Vandermonde matrix; if the last display failed for every such point set, $h$ would interpolate as a degree-$J$ polynomial on arbitrary $J+2$ point sets, hence would have degree at most $J$.  Let $w_i=1/(J+2)$ and choose $\varepsilon>0$ small enough that $w_i\pm\varepsilon\alpha_i\ge0$ for all $i$.  The two laws
\begin{equation*}
    \Pi_1=\sum_i (w_i+\varepsilon\alpha_i)\delta_{x_i},
    \qquad
    \Pi_2=\sum_i (w_i-\varepsilon\alpha_i)\delta_{x_i}
\end{equation*}
have the same first $J$ raw moments and therefore the same $J$-repeat count distribution.  Their $k$th signed moments differ by $2\varepsilon\sum_i\alpha_i h(x_i)\ne0$, proving the claim.
\end{proof}

\redrev{When repeat depth also grows, substituting success frequencies for latent probabilities estimates each fixed Lipschitz signature functional.}
\begin{redrevision}
\begin{proposition}
\label{prop:plugin-omega}
Let $J=J_N$ and define $\widehat q_i=C_i/J_N$.  The plug-in signed measure
\[
    \widehat\omega_{N,J}
    =\frac1N\sum_{i=1}^{N}(2\widehat q_i-1)\,
      \delta_{\widehat q_i(1-\widehat q_i)}
\]
satisfies, for every Lipschitz test function $\varphi$ on $[0,1/4]$,
\begin{equation}
    \E\left|\int\varphi\,\dd\widehat\omega_{N,J}
      -\int\varphi\,\dd\omega_\Pi\right|
    \le
    \frac{\|\varphi\|_\infty}{\sqrt N}
    +\frac{2\|\varphi\|_\infty+\operatorname{Lip}(\varphi)}{2\sqrt J}.
    \label{eq:plugin-bound}
\end{equation}
\redrev{Consequently, for every fixed Lipschitz test function $\varphi$, the integral} $\int\varphi\,\dd\widehat\omega_{N,J}$ converges to $\int\varphi\,\dd\omega_\Pi$ in $L^1$, and therefore in probability, whenever $N\to\infty$ and $J_N\to\infty$.
\end{proposition}
\end{redrevision}

\begin{proof}[Proof of Proposition~\ref{prop:plugin-omega}]
For a test function $\varphi$, set
\begin{equation*}
    h(q)=(2q-1)\varphi(q(1-q)).
\end{equation*}
Then $|h(q)|\le\|\varphi\|_\infty$ and
\begin{equation*}
    |h(q)-h(u)|
    \le\{2\|\varphi\|_\infty+\operatorname{Lip}(\varphi)\}|q-u|,
\end{equation*}
because $q\mapsto q(1-q)$ is one-Lipschitz on $[0,1]$ and $|2q-1|\le1$.  Decompose
\begin{align*}
\E\left|\frac1N\sum_{i=1}^N h(\widehat q_i)-\E h(Q)\right|
&\le
\E\left|\frac1N\sum_{i=1}^N\{h(\widehat q_i)-\E h(\widehat q_i)\}\right| \\
&\quad + \E|h(\widehat q)-h(Q)|.
\end{align*}
The first term is at most $\|\varphi\|_\infty/\sqrt N$ by the variance bound for bounded iid variables.  For the second term,
\begin{equation*}
    \E(|\widehat q-Q|\mid Q)
    \le \sqrt{\operatorname{Var}(\widehat q\mid Q)}
    =\sqrt{Q(1-Q)/J}\le\frac{1}{2\sqrt J}.
\end{equation*}
Combining these inequalities yields \eqref{eq:plugin-bound}.  If $N\to\infty$ and $J_N\to\infty$, the right-hand side tends to zero for each fixed Lipschitz $\varphi$, which gives $L^1$ convergence of the corresponding signed-measure integral and hence convergence in probability.
\end{proof}

\redrev{The complete-curve equivalence rests on polynomial approximation. Its quantitative form controls specified signature functionals on the voting quotient from a finite, imprecise curve, without selecting a latent representative.}
\begin{redrevision}
\begin{proposition}
\label{prop:finite-curve-stability}
Let $\Pi_1,\Pi_2$ be probability laws on $[0,1]$ and suppose that
$\max_{0\le j\le K}|V_j(\Pi_1)-V_j(\Pi_2)|\le\varepsilon$
for an integer $K\ge0$ and $\varepsilon\ge0$.
For every bounded Borel function $\varphi$ on $[0,1/4]$ and every polynomial
$p(r)=\sum_{k=0}^K b_k r^k$,
\begin{equation}
\left|\int\varphi\,\dd\omega_{\Pi_1}
       -\int\varphi\,\dd\omega_{\Pi_2}\right|
\le 2\|\varphi-p\|_\infty
  +2\varepsilon\left(|b_0|+
      \sum_{k=1}^K\frac{|b_k|}{\binom{2k-1}{k}}\right).
\label{eq:finite-curve-stability}
\end{equation}
\end{proposition}

\begin{proof}
Proposition~\ref{prop:signed-signature-realizability} implies
$\|\omega_{\Pi_i}\|_{\mathrm{TV}}\le1$ for $i=1,2$.
The contribution of $\varphi-p$ is therefore at most
$2\|\varphi-p\|_\infty$.
By Corollary~\ref{cor:moment-recovery}, the difference of zeroth moments
is at most $2\varepsilon$, and the difference of $k$th moments is at most
$2\varepsilon/\binom{2k-1}{k}$ for $1\le k\le K$.
Expanding $p$ and applying the triangle inequality proves
\eqref{eq:finite-curve-stability}.
\end{proof}
\end{redrevision}

\redrev{The two terms in \eqref{eq:finite-curve-stability} separate the approximation error from the amplification of observation errors. For each fixed continuous $\varphi$, polynomial approximation makes the first term arbitrarily small at sufficiently large degree; sufficiently accurate curve values then control the second. The bound also applies to any two latent laws compatible with simultaneous error intervals for an observed curve, using their maximum permitted curve separation as $\varepsilon$. Thus it controls the ambiguity of the functional without selecting a latent-law completion. The singular endpoint kernel in Proposition~\ref{prop:endpoint-omega} is unbounded, so its behavior remains governed by the separate threshold conditions in Section~\ref{sec:variation-shape}.}

\section{Multiple-choice plurality as an extension}
\label{app:multiclass-extension}

\redrev{Multiple-choice plurality separates correct-answer probability from the allocation of errors among competing labels. This appendix shows that the latter affects both finite-budget accuracy and the infinite-budget endpoint.}

Binary voting is one-dimensional because every wrong vote is also a vote for the only incorrect alternative.  Multiple-choice plurality has a different geometry.  Let there be $K\ge3$ answer choices.  For a given example, index the correct choice by $0$ and let
\[
    p=(p_0,p_1,\ldots,p_{K-1})\in\Delta_K,
    \qquad
    \Delta_K=\{p\in[0,1]^K:\sum_{j=0}^{K-1}p_j=1\}.
\]
Repeated calls are iid categorical with probabilities $p$.  For $m$ calls, let $N=(N_0,\ldots,N_{K-1})\sim\Mult(m,p)$.  Under uniform tie-breaking among the options attaining the largest count, the conditional plurality accuracy is
\[
    A_{m,K}(p)=
    \E\left[
    \frac{\1\{N_0=\max_j N_j\}}{|\{j:N_j=\max_\ell N_\ell\}|}
    \right].
\]
For a population law $\Gamma$ on $\Delta_K$, the vote curve is $V_{m,K}(\Gamma)=\int A_{m,K}(p)\,\Gamma(\dd p)$.

\begin{proposition}
\label{prop:plurality-endpoint}
For fixed $p\in\Delta_K$,
\begin{equation}
    \lim_{m\to\infty}A_{m,K}(p)
    =
    \frac{\1\{p_0=\max_j p_j\}}{|\{j:p_j=\max_\ell p_\ell\}|}.
    \label{eq:plurality-endpoint}
\end{equation}
Consequently, $V_{m,K}(\Gamma)$ converges to the integral of the right-hand side of \eqref{eq:plurality-endpoint}.
\end{proposition}

The endpoint depends on whether the correct option is the modal option, not merely on the scalar correctness probability $p_0$.  Theorem~\ref{thm:q-not-enough} in Section~\ref{sec:plurality-main} gives a finite-budget obstruction as well; its proof is at the end of this appendix.

Thus a $Q$-only signed measure is not an unrestricted invariant for multiple-choice plurality.  The error allocation over wrong labels is part of the latent geometry.

A one-dimensional theory can still be useful under an explicit wrong-answer symmetry assumption.  Suppose
\begin{equation}
    p_0=q,
    \qquad
    p_1=\cdots=p_{K-1}=\frac{1-q}{K-1}.
    \label{eq:symmetric-wrong}
\end{equation}
Define $A^{\sym}_{m,K}(q)=A_{m,K}(p)$ for $p$ in \eqref{eq:symmetric-wrong}.  Then $V^{\sym}_{m,K}=\E A^{\sym}_{m,K}(Q)$ is determined by the scalar law of $Q$.  The endpoint threshold becomes
\[
    \lim_{m\to\infty}A^{\sym}_{m,K}(q)
    =
    \begin{cases}
    1, & q>1/K,\\
    1/K, & q=1/K,\\
    0, & q<1/K.
    \end{cases}
\]
For $K=2$, this reduces to the binary threshold $q=1/2$ and the signed measure $\omega_\Pi$ gives an exact moment representation.  For $K\ge3$, the symmetric-wrong curve remains a polynomial transform of $Q$, but its increments no longer collapse to the ordinary moments of a single branch variable $q(1-q)$.  The binary signed measure is therefore a special low-dimensional invariant and does not extend generically to plurality voting.

\begin{proof}[Proof of Proposition~\ref{prop:plurality-endpoint}]
By the strong law of large numbers, $N_j/m\to p_j$ almost surely for every $j$.  If $p_0$ is strictly larger than all other coordinates, then eventually $N_0$ is the unique largest count and the tie-broken plurality rule is correct with probability tending to one.  If some wrong coordinate is strictly larger than $p_0$, then eventually a wrong option has a larger count and the accuracy tends to zero.

It remains to handle ties at the maximal probability.  Let $M=\{j:p_j=\max_\ell p_\ell\}$ and let $a=|M|$.  Any class outside $M$ has probability separated below the maximum, so the strong law implies that such classes eventually cannot attain the largest count.  \rev{The unconditional joint count law is invariant under permutations of the equal-probability classes in $M$, and uniform tie-breaking respects that invariance. Thus for every $j\in M$, $\PP(\text{select }j)=\PP(\text{select a class in }M)/a\to1/a$.}  If $0\in M$, the limiting success probability is $1/a$; if $0\notin M$, it is zero.  This gives \eqref{eq:plurality-endpoint}.  The population statement follows from dominated convergence.
\end{proof}

\begin{proof}[Proof of Theorem~\ref{thm:q-not-enough}]
It suffices to compute the displayed $K=3$ vectors.  For $K>3$, append zero coordinates; those extra options have zero probability and therefore do not affect either three-vote plurality or the asymptotic modal comparison.

For $p^{\mathrm{conc}}=(q,1-q,0)$, the third option never appears, so three-vote plurality is binary majority between the correct option and the first wrong option.  Correct plurality requires at least two correct votes, giving \eqref{eq:conc-three}.

For $p^{\mathrm{diff}}=(q,a,a)$ with $a=(1-q)/2$, correct plurality occurs with probability one when the correct option appears two or three times, contributing $q^3+3q^2(1-q)$.  If the correct option appears exactly once and the two wrong options each appear once, all three options tie and uniform tie-breaking chooses the correct one with probability $1/3$.  The probability of counts $(1,1,1)$ is $6qa^2$, so this tie contributes $2qa^2=q(1-q)^2/2$.  All other cases with exactly one correct vote lose.  This proves \eqref{eq:diff-three}.

For $q\in(1/3,1/2)$, the concentrated vector has $1-q>q$, so a wrong option is modal and the endpoint is zero.  The diffuse vector has $q>(1-q)/2$, so the correct option is uniquely modal and the endpoint is one.
\end{proof}

\end{document}